\documentclass{article}

\usepackage{PRIMEarxiv}
\usepackage[utf8]{inputenc} 
\usepackage[T1]{fontenc}    
\usepackage{hyperref}       
\usepackage{url}            
\usepackage{booktabs}       
\usepackage{amsfonts}       
\usepackage{nicefrac}       
\usepackage{microtype}      
\usepackage{fancyhdr}       
\usepackage{graphicx}       

\usepackage{latexsym}
\usepackage{multicol,multirow}
\usepackage{amsmath,amssymb}
\usepackage{mathrsfs}
\usepackage{amsthm}
\usepackage{rotating}
\usepackage{appendix}
\usepackage[authoryear]{natbib}
\usepackage{ifpdf}
\usepackage{times}
\usepackage{newtxmath}
\usepackage{textcomp}
\usepackage[table,xcdraw]{xcolor}
\usepackage{footnote}
\makesavenoteenv{tabular}

\newtheorem{proposition}{Proposition}

\title{Emulating Non-Differentiable Metrics via Knowledge-Guided Learning: Introducing the Minkowski Image Loss}

\author{
  Filippo Quarenghi\thanks{Corresponding author. \texttt{filippo.quarenghi@gmail.com}} \\
  Faculty of Geosciences and Environment \\
  Expertise Center for Climate Extremes (ECCE) \\
  University of Lausanne \\
  Lausanne, VD, Switzerland
  \And
  Ryan Cotsakis \\
  Expertise Center for Climate Extremes (ECCE) \\
  University of Lausanne \\
  Lausanne, VD, Switzerland
  \And
  Tom Beucler \\
  Faculty of Geosciences and Environment \\
  Expertise Center for Climate Extremes (ECCE) \\
  University of Lausanne \\
  Lausanne, VD, Switzerland
}

\begin{document}
\maketitle

\begin{abstract}
The ``differentiability gap'' presents a primary bottleneck in Earth system deep learning: since models cannot be trained directly on non-differentiable scientific metrics and must rely on smooth proxies (e.g., MSE), they often fail to capture high-frequency details, yielding ``blurry'' outputs. We develop a framework that bridges this gap using two different methods to deal with non-differentiable functions: the first is to analytically approximate the original non-differentiable function into a differentiable equivalent one; the second is to learn differentiable surrogates for scientific functionals. We formulate the analytical approximation by relaxing discrete topological operations using temperature-controlled sigmoids and continuous logical operators. Conversely, our neural emulator uses Lipschitz-convolutional neural networks to stabilize gradient learning via: (1) spectral normalization to bound the Lipschitz constant; and (2) hard architectural constraints enforcing geometric principles. We demonstrate this framework's utility by developing the Minkowski image loss, a differentiable equivalent for the integral-geometric measures of surface precipitation fields (area, perimeter, connected components). Validated on the EUMETNET OPERA dataset, our constrained neural surrogate achieves high emulation accuracy, completely eliminating the geometric violations observed in unconstrained baselines, which generate physically impossible precipitation fields in up to 7.8\% of cases. However, applying these differentiable surrogates to a deterministic super-resolution task reveals a fundamental trade-off: while strict Lipschitz regularization ensures optimization stability, it inherently over-smooths gradient signals, restricting the recovery of highly localized convective textures. This work highlights the necessity of coupling such topological constraints with stochastic generative architectures to achieve full morphological realism.
\end{abstract}

\keywords{differentiability \and knowledge-guided machine learning \and Minkowski functionals \and radar precipitation \and super-resolution}

\section{Introduction}

Structural realism in predicted precipitation fields depends on metrics that are inherently non-differentiable, creating a bottleneck for gradient-based optimization. To bridge this gap, we leverage knowledge from integral geometry, which reveals organizational biases often missed by standard statistical verification \citep{https://doi.org/10.1029/2010JD014741}. Unlike recent work exploring surrogates in latent spaces \citep{patel2020learningsurrogatesdeepembedding}, we enforce these properties directly in the optimization landscape. We introduce a framework situated at the intersection of differentiable programming \citep{gmd-16-3123-2023} and knowledge-guided machine learning \citep{Karpatne2017TheoryGuided}. Within this framework, we compare an analytical approximation against a neural emulator, both designed to map the level sets of precipitation fields to key integral-geometric measures (area, perimeter, and number of connected components). By relaxing this discrete combinatorial problem into a continuous regression task \citep{Hu2019TopologyPreserving}, we enable the direct backpropagation of structural error signals. Crucially, we ensure emulation stability adopting a knowledge-guided Lipschitz-convolutional neural network (Lip-CNN) \citep{Gouk2021Regularisation}. This architecture combines spectral normalization---to prevent exploding gradients---with hard geometric constraints (monotonicity, isoperimetry) embedded in the computation graph.

Concurrent efforts to integrate non-differentiable operations into automatic differentiation frameworks, such as the SoftJAX and SoftTorch libraries \citep{paulus2026softjaxsofttorchempowering}, provide general-purpose continuous relaxations for elemental primitives (e.g., thresholding, fuzzy logic). While these drop-in replacements facilitate the formulation of analytical approximations akin to our continuous baseline, cascading numerous element-wise soft operators through complex topological algorithms frequently exacerbates gradient attenuation and mathematical entanglement. In contrast, our Lipschitz-regularized neural emulation circumvents the need for explicit operator relaxation by learning a holistic, topologically informed mapping that ensures robust structural fidelity and stable gradient flow.

Our primary contributions are: (1) proposing a general framework for injecting non-differentiable domain knowledge into training via learned surrogates, compared against standard analytical approximations; (2) introducing the Minkowski image loss, a differentiable objective that penalizes errors in rigid-motion invariants (area, perimeter, connected components); and (3) demonstrating that hard geometric constraints paired with Lipschitz regularization prevents violating geometric principles and overfitting observed in unconstrained baselines.


\begin{figure}[htbp]
    \centering
    \includegraphics[width=.68\linewidth]{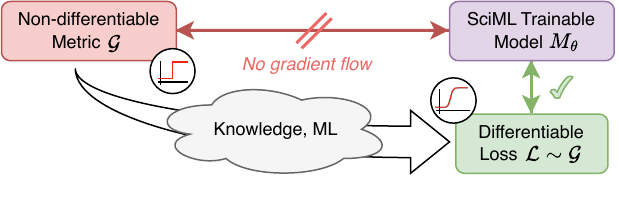}
    \caption{Conceptual overview of the differentiable surrogate training framework. The lack of gradient flow (red connection) prevents the direct optimization of SciML models ($M_\theta$) using non-differentiable scientific metrics ($\mathcal{G}$), which typically involve discrete operations like thresholding (step function plot). We overcome this limitation by using knowledge and machine learning techniques to create a differentiable surrogate loss ($\mathcal{L}$) that approximates the target metric. This differentiable loss, with a well-behaved optimization landscape, can then be successfully integrated into the model's training pipeline (green connection), enabling the direct backpropagation of structurally and physically relevant error signals}   
    \label{fig:Fig1}
\end{figure}

\section{Methodology}

\subsection{Approaches to Non-Differentiable Metrics}
\label{sec:emulation_method}

Scientific evaluation often relies on functionals $f: \mathcal{X} \rightarrow \mathcal{Y}$ involving discrete operations such as counting, thresholding, or binning. These operations are inherently non-differentiable; they manifest as piecewise constant functions whose gradient $\nabla_{\mathbf{x}} f$ is zero almost everywhere or undefined at jump discontinuities (Dirac deltas), producing a loss surface of flat plateaus that prevents the backpropagation of error signals required to update an upstream model. Two primary methodologies exist to circumvent this limitation. The first relies on \textit{analytical approximation}, wherein discrete operations are substituted with continuous, smooth relaxations yielding a closed-form differentiable surrogate. The second leverages \textit{neural emulation}, employing an artificial neural network to map the input space directly to the target metric; because neural networks are differentiable by construction, this learned surrogate integrates seamlessly into the computational graph.

\subsubsection{Analytical Approximation of Non-differentiable Metrics}
The analytical approximation strategy aims to replace non-differentiable operations with smooth variants that permit stable gradient flow. For example, binary thresholding—mathematically represented by the Heaviside step function $H(x)$—is commonly relaxed using a logistic sigmoid function $\sigma(x / \tau)$. Here, $\tau > 0$ acts as a temperature parameter controlling the steepness of the approximation. As $\tau \to 0$, the continuous function converges to the exact discrete behavior, but its gradients vanish almost everywhere; conversely, larger values of $\tau$ provide stable gradients at the cost of approximation fidelity. Consequently, $\tau$ is often annealed dynamically during optimization. Similar continuous relaxations are employed for discrete selection operations, such as replacing $\arg\max$ with a differentiable softmax approximation \citep{jang2017categoricalreparameterizationgumbelsoftmax}. While these analytical substitutions preserve the mechanistic interpretability of the original metric, deriving stable, closed-form relaxations for complex, multi-step spatial metrics is not always mathematically possible.

\subsubsection{Emulation of Non-Differentiable Metrics: Surrogate Modeling and Lipschitz Regularization} \label{sec:emulation}
To restore differentiability, we approximate the discrete map $f$ with a parameterized neural surrogate $f_\theta: \mathcal{X} \rightarrow \mathcal{Y}$. Minimizing the regression error on static data is insufficient: For $f_\theta$ to serve as a stable loss function, it must provide ``well-behaved'' gradients that are informative across the input domain $\mathcal{D}$.

\textit{Gradient Propagation}: A critical requirement is non-vanishing input sensitivity. To avoid the ``dead zones'' associated with rectifying activations (e.g., ReLU where $\nabla=0$ for $x < 0$), we employ smooth, $C^\infty$ activation functions---specifically the Gaussian error linear unit \citep[GELU;][]{Hendrycks2016GELU}. Furthermore, to facilitate gradient flow through deep architectures, we utilize residual connections $\mathbf{y} = \mathbf{x} + F(\mathbf{x})$, allowing error signals to propagate unimpeded through the identity path.

\textit{Manifold Smoothing}: Early in training, the upstream model may yield a distribution $P_{\mathrm g}$ far from the data distribution $P_{\text{data}}$; if trained only on $P_{\text{data}}$'s support, the emulator can yield arbitrary gradients for $\mathbf{x}\sim P_{\mathrm g}$. We hence inject Gaussian noise $\boldsymbol{\eta}\sim\mathcal{N}(\mathbf{0},\sigma^2\mathbf{I})$ during emulator training, effectively convolving the target functional with a smoothing kernel and improving robustness to off-manifold queries. (In practice, we found this step unnecessary for the OPERA dataset given the mix-up augmentation described in Section 3.1, and set $\sigma^2 = 0$ in all reported experiments.)

\textit{Lipschitz Constraints}: Finally, we bound the gradient norm to prevent optimization instability. Unconstrained neural networks can learn arbitrarily sharp gradients, leading to exploding loss terms. We enforce stability by constraining the Lipschitz constant $K$ of the network mapping. While soft constraints like the gradient penalty \citep{Gulrajani2017WGAN} are common in adversarial training, they do not guarantee bounds on unseen data. Instead, we enforce a hard constraint via spectral normalization \citep{Miyato2018Spectral} on all weight matrices $W_l$. This ensures that the global Lipschitz constant is bounded by the product of the layer-wise spectral norms: $ \| f_{\theta}(\mathbf{x}) - f_{\theta}(\mathbf{y}) \|_2 \leq K \| \mathbf{x} - \mathbf{y} \|_2$, implying  $\sup_{\mathbf{x} \in \mathcal{D}} \| \nabla_{\mathbf{x}} f_{\theta}(\mathbf{x})\|_2 \leq K$, almost everywhere in $\mathcal D$, where typically $K \approx 1$ and $\|.\|_2$ is the $L^2$ norm. While fixing $K$ prior to training limits the network's expressivity if the true functional varies rapidly \citep{Ducotterd2022Lipschitz}, for the integral-geometric measures considered in Section~\ref{sec:minkowski_loss}, empirical results suggest that the stability benefits of a hard $K=1$ constraint outweigh the theoretical expressivity loss.

\subsection{Minkowski Image Loss: Injecting Knowledge from Integral Geometry}
\label{sec:minkowski_loss}

\subsubsection{Integral-Geometric Characterization via Minkowski Functionals}
To capture the structural properties of physical fields beyond pixel-wise intensity, we adopt a framework based on the \textit{Minkowski functionals} from integral geometry \citep{schneiderStochastic2008}. For a compact set $K \subset \mathbb{R}^2$ with a sufficiently smooth boundary (see Appendix~\ref{sec:minkowski}), there exist three Minkowski functionals: the area, half the perimeter, and the Euler characteristic. These functionals are denoted $A_K$, $P_{K/2}$, and $\chi_K$, respectively. The Euler characteristic describes the connectivity of the set and is defined in 2D as $\chi_K = \mathrm{CC}_K - H_K$, where $\mathrm{CC}_K$ is the number of connected components and $H_K$ is the number of holes. The suitability of these measures for a loss function is rigorously grounded in Hadwiger’s Characterization Theorem \citep{Hadwiger1957}, which establishes that any functional on $K$ satisfying rigid-motion invariance, continuity, and additivity can be expressed as a linear combination of these three measures. These mathematical properties map directly to critical advantages in neural network training: (i) \textit{Rigid-motion invariance:} Decouples the loss from the exact spatial positioning of extreme events, mitigating the ``double penalty'' problem \citep{Ebert2008Fuzzy} inherent in pixel-wise metrics where a slight displacement is penalized twice (as a miss and a false alarm). (ii) \textit{Continuity:} Ensures gradient stability, as minor shape deformations result in commensurate, small changes in the loss value, unlike discontinuous metrics (e.g., pure object counting) which exhibit step functions.

\subsubsection{The $\gamma$-vector and Minkowski Image Loss Formulation}
A continuous physical field is transformed into $N$ binary images (excursion sets) by thresholding at $N$ intensity levels, $u^{(1)}, \dots, u^{(N)}$. For each excursion set, we compute the area ($A$), perimeter ($P$), and the number of connected components ($\mathrm{CC}$). The values of these functionals for the $i$-th threshold are denoted $A^{(i)}$, $P^{(i)}$, and $\mathrm{CC}^{(i)}$. We concatenate these statistics into a single non-negative vector $\boldsymbol{\gamma}$ that provides a multi-level integral-geometric summary of the field:
\begin{align}
    \label{eqn:gamma}
    \boldsymbol{\gamma} = [A^{(1)}, P^{(1)}, \mathrm{CC}^{(1)}, \ldots, A^{(N)}, P^{(N)}, \mathrm{CC}^{(N)}]^\top \in \mathbb{R}^{3N}.
\end{align}
The distance between two $\boldsymbol{\gamma}$-vectors serves as a proxy for the structural divergence between the predicted and target fields. Integrating this metric into training ensures the network minimizes geometric error alongside pixel-wise error. The reader is referred to Appendix~\ref{appendix:gamma-targets} for further details regarding the computation of the $\gamma$-vector from data.

A significant challenge is the order-of-magnitude disparity between vector components; $A^{(i)}$ (pixel summation) can be orders of magnitude larger than $\mathrm{CC}^{(i)}$ (sparse count). Without normalization, gradients would be dominated by area terms, ignoring topological structure.
To address this, we define the \textit{Minkowski image loss} $\mathcal{M}$ using a log-transformed $L^1$ distance. This transformation effectively normalizes the scale differences, operating as a relative error metric that balances the contribution of large-magnitude features (area) and small-magnitude features (connectivity):
\begin{equation}
    \label{eq:geo_loss}
    \mathcal{M}(\hat{\boldsymbol{\gamma}}, \boldsymbol{\gamma}) = \| \log(\hat{\boldsymbol{\gamma}} + \mathbf{1}) - \log(\boldsymbol{\gamma} + \mathbf{1}) \|_1
\end{equation}
where $\mathbf{1}$ is a vector of ones to ensure numerical stability and $\|.\|_1$ is the $L^1$ norm. This Minkowski image loss functions as a regularizer alongside standard pixel-wise metrics (e.g., MSE). While pixel loss enforces local intensity accuracy, $\mathcal{M}$ acts as a global structural constraint. Since the direct computation of $\boldsymbol{\gamma}$ involves non-differentiable operations (thresholding, counting) that break the computation graph, we employ the knowledge-guided differentiable surrogate described in Section~\ref{sec:emulation_method}. This surrogate allows the error signal $\nabla_{\hat{\boldsymbol{\gamma}}} \mathcal{M}$ to be backpropagated through the emulator to update the trainable model.

\subsection{Application to Integral-Geometric Measures}

While analytical approximations directly relax discrete metrics, formulating closed-form solutions for complex topologies proves computationally restrictive and prone to vanishing gradients. As a robust alternative, we introduce a neural emulator as a learned differentiable surrogate, approximating the non-differentiable mapping $F: \mathcal{X} \rightarrow \mathbb R^{3N}$ from physical fields to $\boldsymbol{\gamma}$-vectors. Standard pixel-wise metrics (e.g., MSE) inherently smooth the multi-scale spatial organization critical to precipitation modeling. By embedding this emulator into the loss function, we directly penalize geometric errors. We train the model to regress exact ground-truth descriptors (area, perimeter, connected components) generated via \texttt{GUDHI} \citep{GUDHI2014}. The bipartite architecture pairs a Lipschitz-regularized feature extractor with geometrically constrained heads to decode valid $\boldsymbol{\gamma}$-vector components.

\subsubsection{Analytical Approximation of Integral-Geometric Measures} \label{sec:analapprox}

Our analytical approximation substitutes the non-differentiable Heaviside step function with a temperature-controlled sigmoid, mapping the continuous physical field into a pseudo-probability space $P\in(0,1)$ to enable differentiable tensor operations. We calculate area as the weighted sum of $P$ and derive perimeter from finite-difference approximations of spatial derivatives. To compute the Euler characteristic (a proxy for connected components), we evaluate the alternating sum of grid vertices, edges, and faces, replacing discrete boolean logic with the G\"odel t-norm (\texttt{min} operator) to ensure continuous gradient flow. To stabilize these topological metrics against minor fluctuations, we preprocess the input tensor using differentiable morphological filters and a local persistence threshold mask, ultimately aggregating and scaling the measures via a symmetric logarithm.

\subsubsection{Emulation of Integral-Geometric Measures via Knowledge-Guided Learning}

\begin{figure}[]
    \centering
    \includegraphics[width=\linewidth]{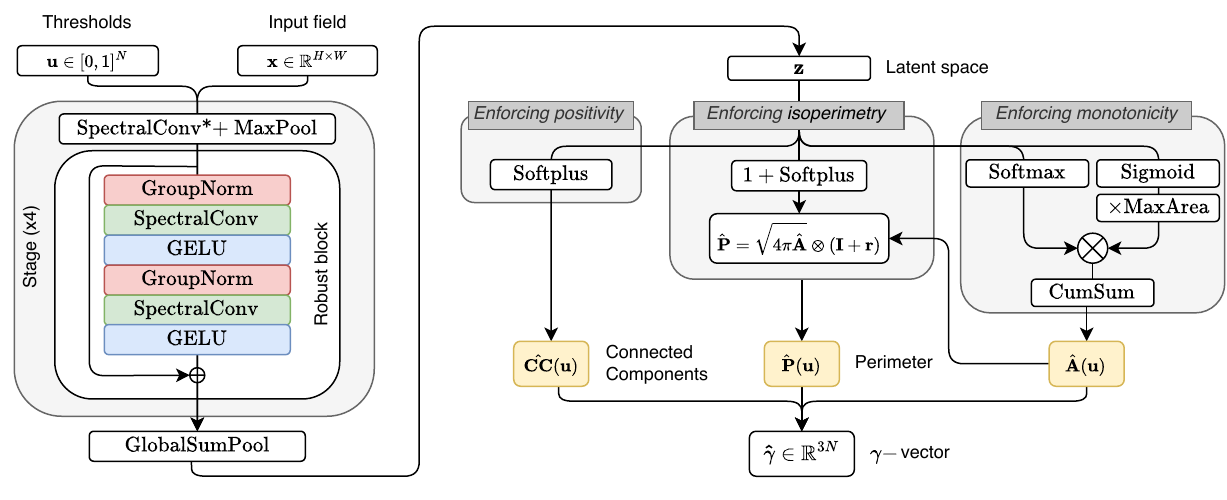}
    \caption{\textbf{(Left) Lipschitz-Bound Feature Extractor:} The backbone acts as a hierarchical encoder. Stability is enforced via spectral normalization and residual connections. We use global sum pooling to preserve the extensivity of geometric features. \textbf{(Right) Geometric Constraint Heads:} The Area $\hat{A}(u)$ is constructed via integration of a predicted probability density to enforce monotonicity. The perimeter $\hat{P}(u)$ is conditioned on the area via the isoperimetric inequality, preventing geometric principles' violations}
    \label{fig:arch}
\end{figure}

\paragraph{Feature Extractor with Lip-CNN}
To map the physical field $\mathbf{x} \in \mathbb{R}^{H \times W}$ to the latent geometric manifold $\mathbf{z}$, we employ a CNN encoder designed to respect the mathematical constraints of the $\boldsymbol{\gamma}$-vector. The input field is normalized to $[0, 1]$ via a fixed scaling factor. The backbone processes the field through four hierarchical stages (Figure~\ref{fig:arch}, left), with spectral normalization \citep{Miyato2018Spectral} applied to all convolutional kernels except the first layer, which is left unconstrained to allow unrestricted input scaling. Deep feature extraction is handled by ``robust blocks'' containing residual connections and group normalization, ensuring invariance to batch statistics. We use global sum pooling rather than average pooling, since area and connected components are extensive quantities whose additivity must be preserved as the latent representation scales with domain size.

\paragraph{Knowledge-Guided Geometric Constraints}
\label{sec:constrained}
Unlike standard multi-task regression where outputs are independent, our architecture couples the output heads to strictly enforce geometric consistency.

\textit{Area Head (Monotonicity via Integration).} Directly regressing the area metric $A^{(i)}$ for each threshold fails to guarantee the monotonicity property $A^{(i)} \ge A^{(i+1)}$ (as threshold increases, excursion set area decreases). To resolve this, we use the fact that the area function is effectively the cumulative distribution function of the field intensities. The network predicts a discrete probability density using a softmax layer and a scalar total area $A_{\text{total}}$. The final area vector $\mathbf{\hat{A}}$ is computed via the cumulative sum of these probabilities. This inductive bias ensures that the predicted area is strictly monotonic by construction.

\textit{Perimeter Head (Isoperimetric Inequality).} We strictly condition the perimeter on the area by predicting a dimensionless ``roughness'' coefficient $\hat{\mathbf{r}} \ge 0$ (via a Softplus activation) rather than the absolute perimeter. The perimeter is then reconstructed as $ \mathbf{\hat{P}} = \sqrt{4\pi\mathbf{\hat{A}}} \odot (\mathbf{1} + \mathbf{\hat{r}})$. This hard-codes the isoperimetric inequality ($P^2 \ge 4\pi A$) into the forward pass. It prevents the model from generating a perimeter smaller than that of a perfect circle with equivalent area, ensuring 0\% geometrical violations.

\textit{Topological Head (Positivity).} The number of connected components $\mathrm{CC}$ is a counting measure. To reflect this, the output is rectified using Softplus ($\log(1+e^x)$), ensuring $\widehat{\mathrm{CC}} \in \mathbb{R}_+$ while maintaining differentiable non-zero gradients, unlike a hard ReLU clip which may lead to ``dead neurons''.

\paragraph{Training Setup} Motivated by the target data's statistical properties, we optimize $\theta$ using the Minkowski image loss: $ \mathcal{L}_{\text{train}} = \mathcal{M}(\hat{\boldsymbol{\gamma}}_{\theta}(\mathbf{x}), \boldsymbol{\gamma}_{\text{true}})$. Logarithmic scaling prevents large disparities between area and connected components, while the $L^1$ norm ensures stability against outliers from complex storm structures.

\section{Application: Observed Instantaneous Surface Precipitation Fields}

Emulating Minkowski functionals is particularly relevant for precipitation, which---unlike smooth fields---exhibits high intermittency, non-Gaussian statistics, and complex multi-scale geometries. Standard pixel-wise losses often fail to capture these morphological properties, leading to structurally amorphous (``blurred'') predictions. By applying our Minkowski image loss to radar-based precipitation estimates, we aim to encourage realistic structures in generative tasks, specifically km-scale downscaling.

\subsection{Dataset and Preprocessing}

We validate  our approach using the pan-European EUMETNET radar composites provided by the European Operational Program for Exchange of Weather Radar Information (OPERA) OPERA \citep{Saltikoff2019OPERA}. We utilize the instantaneous surface rain rate product, leveraging high-definition streams (ODISSEY/NIMBUS) with 15-minute updates. Spanning Aug 1, 2023 to Oct 30, 2024, this archive yields >1M patches of 128$\times$128 pixels (256$\times$256\,km) across diverse regimes, from localized summer convection to large-scale winter synoptic storms. We partition the dataset chronologically to preclude temporal autocorrelation: data from August 2023 through June 2024 is utilized for training, July 2024 for validation, and August through October 2024 is reserved for independent testing. We apply a drizzle threshold $\delta = 0.1 \, \mathrm{mm\,h^{-1}}$; intensity values below $\delta$ are zeroed out to remove sensor noise. For training, data undergoes a log-linear transformation to compress the dynamic range: $x_{\text{norm}} = \log(1+x)/S$, where $S$ is the global maximum of the log-transformed training set. To enforce robustness, we apply an offline ``mix-up'' strategy in physical space. We generate synthetic convex combinations of ground truth fields and noise-injected interpolated fields to ensure off-manifold robustness: $\mathbf{x}_{\text{mix}} = \lambda \mathbf{x}_{\text{real}} + (1 - \lambda)(\mathbf{x}_{\text{interp}} + \boldsymbol{\epsilon})$ where $\lambda \sim \text{Beta}(\alpha, \alpha)$ and $\boldsymbol{\epsilon} \sim \mathcal{N}(\mathbf{0}, \sigma^2 \mathbf{I})$.

\subsection{Evaluated Architectures and Baselines}

\noindent \textbf{Statistical baseline: bicubic interpolation} As a naive deterministic baseline that preserves total domain mass without introducing high-frequency artifacts, we employ bicubic interpolation. This non-parametric approach acts as a lower bound for structural fidelity, isolating the performance gains achievable purely through the spatial redistribution of the low-resolution input.

\noindent \textbf{Analytical approximation} To evaluate the efficacy of the proposed continuous relaxation, we integrate the analytical surrogate (detailed in Section \ref{sec:analapprox}) into the optimization loop. This model operates on identical inputs but replaces the learned emulation heads with differentiable temperature-controlled sigmoid and G\"odel t-norm operations.

\noindent \textbf{Surrogate CNN emulators ablation study} To quantify the impact of embedding domain knowledge, we perform an ablation study comparing three distinct architectures. This hierarchy disentangles the benefits of gradient stabilization (Lipschitz continuity) from the benefits of hard geometric constraints.

\textit{Unconstrained CNN}: A standard convolutional neural network employing ReLU activations, batch normalization, and global average pooling. This represents a naive deep learning baseline. It treats the components of the $\boldsymbol{\gamma}$-vector ($A, P, \mathrm{CC}$) as independent regression targets.
    
\textit{Unconstrained Lip-CNNs} use our specialized backbone (spectral normalization, group normalization, residual layers, global sum pooling) to ensure gradient stability but retain standard, independent regression heads. This isolates the benefit of Lipschitz regularization but lacks structural constraints; it remains capable of predicting geometrically impossible states (e.g., non-monotonic area curves).
    
\textit{Constrained Lip-CNN}: The full architecture described in Section \ref{sec:constrained}. This model combines the Lipschitz-regularized backbone with the geometrically constrained output heads, structurally enforcing monotonicity, non-negativity, and isoperimetric consistency.

\subsection{Verification via Diagnostic Input Optimization}
We assess the utility of the emulator's gradients using a diagnostic feature inversion test. This procedure involves freezing the emulator weights $\theta$ and iteratively updating a noise vector $\mathbf{x}$ to minimize the distance to a target metric value $\mathbf{y}_{\text{target}}$: $\mathbf{x}_{t+1} \leftarrow \mathbf{x}_t - \alpha \nabla_{\mathbf{x}} \| f_{\theta}(\mathbf{x}_t) - \mathbf{y}_{\text{target}} \|_2$. This acts as a unit test for the gradient field. If $f_\theta$ has learned structurally robust features, the optimization recovers an input $\mathbf{x}^*$ consistent with domain physics (e.g., connected precipitation blobs). If the optimization produces high-frequency adversarial noise, it indicates the gradients are ill-conditioned, requiring stricter regularization.

\subsection{Application of Minkowski image loss to Super-resolution of Instantaneous Surface Rain Rate}

To evaluate our framework, we formulate a perfect-model super-resolution task for instantaneous surface rain rates from the OPERA dataset, emulating a transition from coarse global circulation models ($\sim$25 km) to high-resolution radar (2 km). We generate low-resolution inputs via $12.5\times$ block-averaging, concatenating them with a normalized digital elevation model to condition the downscaling on orographic forcing. Target fields undergo a normalized log-space transformation to stabilize optimization across highly skewed precipitation distributions. We benchmark a deterministic log-space residual U-Net against a conditional denoising diffusion probabilistic model (DDPM), evaluated exclusively via deterministic DDIM sampling to ensure rigorous comparison. We test whether augmenting the mean squared error (MSE) with the Minkowski image loss (via analytical approximation or neural emulation) imparts generative-level structural realism to deterministic architectures, thereby circumventing the substantial computational sampling costs of diffusion models. Appendix \ref{AppendixSR} details the implementation.

\section{Results}
\subsection{Accuracy and Geometric Fidelity of Analytical Approximation and Emulation}

\begin{table}
    \centering
    \small
    \caption{\textbf{Comparison of analytical approximation and emulation.} Performance is measured via the Minkowski image loss $\mathcal{M}$ (lower is better), the coefficient of determination $R^2$ for the full vector and individual components (Area (A), Perimeter (P), Connected Components (CC)), and the rate of isoperimetric violations ($\nu_{\text{Iso}}$).}
    \begin{tabular}{lcccccc}
        \toprule
        \textbf{Architecture} & $\mathcal{M}$ ($\downarrow$) & $R^2$ ($\uparrow$) & $R^2_A$ & $R^2_P$ & $R^2_{\mathrm{CC}}$ & $\nu_{\text{Iso}}$ (\%) \\
        \midrule
        PCR & 5.36 & 0.35 & 0.46 & 0.39 & 0.25 & 27.0 \\
        Analytical Approx. & \textbf{1.03} & 0.67 & \textbf{0.99} & 0.92 & 0.14 & \textbf{0.3} \\
        CNN (Unconstrained) & 5.77 $\pm$ 0.02 & 0.95 $\pm$ 0.02 & 0.94 $\pm$ 0.02 & 0.94 $\pm$ 0.01 & 0.96 $\pm$ 0.02 & 7.8 $\pm$ 0.3 \\
        Lip-CNN (Unconstrained) & 5.74 $\pm$ 0.01 & 0.88 $\pm$ 0.03 & 0.79 $\pm$ 0.03 & 0.91 $\pm$ 0.01 & 0.95 $\pm$ 0.03 & 5.8 $\pm$ 0.2 \\
        \textbf{Lip-CNN (Constrained)} & 5.73 $\pm$ 0.01 & \textbf{0.99 $\pm$ 0.01} & \textbf{0.99 $\pm$ 0.01} & \textbf{0.99 $\pm$ 0.01} & \textbf{0.97 $\pm$ 0.02} & \textbf{0.00} \\
        \bottomrule
    \end{tabular}
    \label{tab:emulator-comparison}
\end{table}

To establish performance bounds, we evaluate a linear principal component regression (PCR) baseline. PCR demonstrates poor predictive capability (overall $R^2 = 0.35$) and a high isoperimetric violation rate ($\nu_{\text{Iso}} = 27.0\%$), confirming the highly non-linear mapping from physical fields to integral-geometric measures (Table~\ref{tab:emulator-comparison}). Conversely, the analytical approximation achieves the lowest Minkowski image loss ($\mathcal{M} = 1.03$) and near-zero isoperimetric violations (0.3\%). While perfectly capturing area ($R^2_A = 0.99$) and approximating perimeter well ($R^2_P = 0.92$), its overall $R^2$ (0.67) is hindered by a near-complete failure to predict connected components ($R^2_{\mathrm{CC}} = 0.14$; see Appendix~\ref{Appendix:Chi}). 

While all neural architectures converged, the constrained Lip-CNN achieves the best performance across all metrics, definitively eliminating isoperimetric violations ($\nu_{\text{Iso}} = 0\%$). The unconstrained Lip-CNN underperforms the vanilla CNN baseline; while $\mathcal{M}$ remains comparable, $R^2$ scores degrade significantly (e.g., area $0.94 \rightarrow 0.79$). Imposing spectral normalization bounds the Lipschitz constant, limiting representation power. Without geometric constraint heads, this ``stiffened'' backbone struggles to fit high-variance data. Pairing the Lipschitz backbone with geometric constraints (Lip-CNN Constr.) recovers and surpasses baseline performance. Architectural inductive biases (monotonicity and isoperimetry) effectively reduce hypothesis space complexity, acting as a scaffold that allows the stable backbone to converge to a physically valid solution despite reduced capacity. Appendix~\ref{appendix:eval_emu} provides further evaluation and examples of the emulators' capabilities.

\subsection{Diagnostic Inversion and Gradient Stability}

\begin{figure} [h]
    \centering
    \includegraphics[width=0.95\linewidth]{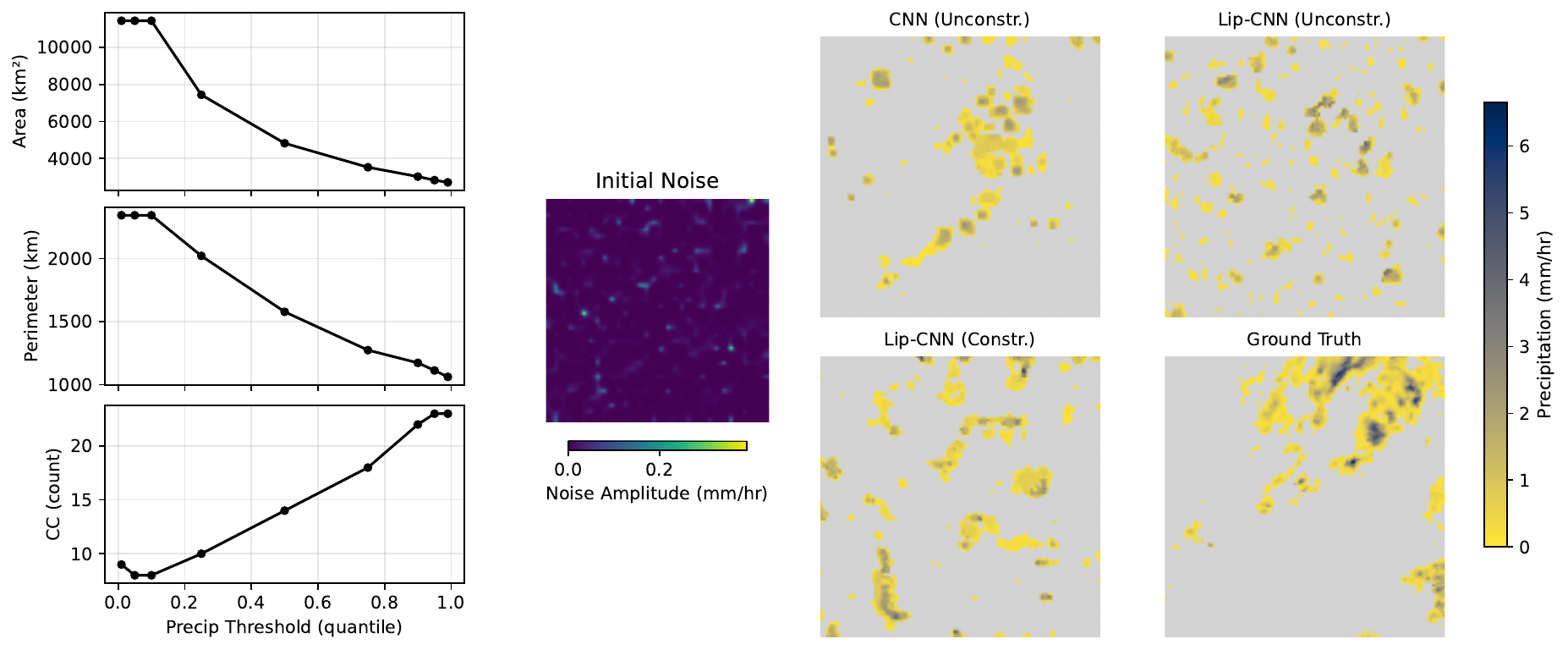}
    \caption{\textbf{Feature Inversion Results.} Qualitative comparison of precipitation fields reconstructed by inverting the target Minkowski vector $\boldsymbol\gamma_\mathrm{gt}$ for a test sample (left). All architectures recover the storm's magnitude, while only constrained models generates coherent, smooth intensity gradients characteristic of convective cells (bottom left)}
    \label{fig:inversion}
\end{figure}

Figure~\ref{fig:inversion} demonstrates that both the \textit{vanilla CNN} and \textit{constrained Lip-CNN} reconstruct coherent precipitation structures from geometric signatures, confirming Minkowski functionals encode sufficient morphological information to guide optimization from Gaussian noise to a localized storm object. However, textural quality diverges: the \textit{vanilla CNN} produces disjointed, "blob-like" structures with sharp, artificial boundaries, whereas Lipschitz-regularized models generate continuous intensity gradients that naturally reproduce the smooth decay from convective cores to stratiform peripheries. In appendices~\ref{appendix:eval_inv} and~\ref{appendix:eval_grads} we report further evaluation of the constrained Lip-CNN emulator's gradient quality.

\subsection{Deterministic super-resolution}

\begin{figure}
    \centering
    \includegraphics[width=.9\linewidth]{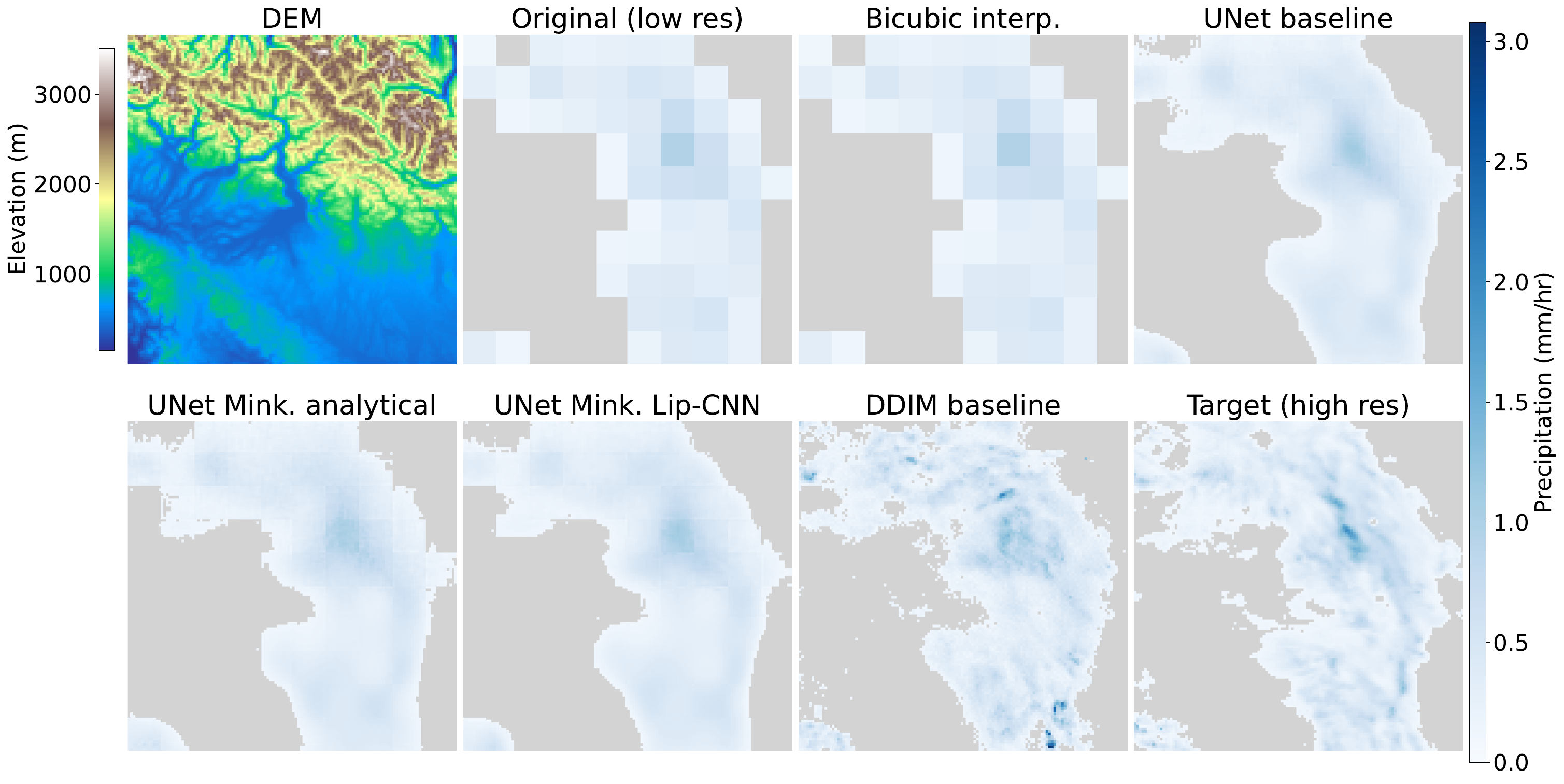}
    \caption{\textbf{Structural fidelity in precipitation downscaling.} The unconstrained UNet produces an amorphous shield, missing localized convective peaks. Integral-geometric constraints (analytical and Lip-CNN) yield only marginal improvement. The stochastic DDIM baseline recovers textural realism and multi-scale variance, confirming the limitations of deterministic optimization}    
    \label{fig:srcomparison}
\end{figure}

\begin{table}
    \centering
    \small
    \caption{\textbf{Quantitative comparison of super-resolution deterministic architectures.} Performance is measured via the Minkowski image loss ($\mathcal{M}$), the mean absolute error (MAE), the structure-amplitude-location (SAL) metric, and the radially averaged power spectral density (RAPS) error.}
    \begin{tabular}{lcccccc}
        \toprule
        \textbf{Architecture} & $\mathcal{M}$ ($\downarrow$) & $MAE$ ($\downarrow$) & SAL & RAPS ($\downarrow$) \\
        \midrule
        Interp.                                                 & 3.44 & 0.06 & 0.37, \textbf{-0.01}, 0.06 & 1.04 \\
        UNet                                                    & 3.21 & \textbf{0.05} & 0.32, 0.09, 0.15 & 2.88 \\
        UNet + $\mathcal{M}_{\text{Ana}}$                       & 2.74 & \textbf{0.05} & 0.41, -0.20, \textbf{0.03} & 1.37 \\
        UNet + $\mathcal{M}_{\text{Emu, Lip-CNN (Constr.)}}$    & 3.25 & \textbf{0.05} & 0.31, -0.29, \textbf{0.07} & 1.78 \\
        DDIM                                                    & 2.96 & 0.09 & \textbf{-0.08}, 0.26, 0.17 & \textbf{0.97} \\
        \bottomrule
    \end{tabular}
    \label{tab:sr-comparison}
\end{table}

Table~\ref{tab:sr-comparison} details super-resolution performance metrics. The bicubic interpolation baseline exhibits seemingly strong spectral fidelity (RAPS = 1.04). However, this metric is an artifact of spatial smoothing; interpolation merely redistributes low-frequency kinetic energy without generating physical high-frequency features. This failure is captured by its high topological error ($\mathcal{M} = 3.44$) and blocky visual structure (Figure~\ref{fig:srcomparison}).

Comparing the baseline UNet to topologically constrained architectures illustrates deterministic optimization challenges. The unconstrained UNet minimizes spatial distortion (MAE = 0.05) but actively dampens spatial variance. Figure~\ref{fig:srcomparison} confirms this: the UNet predicts an amorphous precipitation shield, failing to resolve localized convective cores, yielding the highest spectral error (RAPS = 2.88). Augmenting the UNet with the analytical Minkowski loss ($\mathcal{M}_{\text{Ana}}$) quantitatively improves RAPS (1.37) and topological error ($\mathcal{M}=2.74$) without compromising MAE (0.05). However, visually, both analytical and emulator-constrained ($\mathcal{M}_{\text{Emu, Lip-CNN (Constr.)}}$) outputs only marginally improve upon the baseline. The fields remain artificially smooth, indicating the current surrogate loss implementation insufficiently enforces morphological realism. 

Benchmarked against the stochastic DDIM, a distinct trade-off emerges. The DDIM achieves the most realistic energy spectrum (RAPS = 0.97) and structure amplitude ($S=-0.08$), successfully synthesizing high-frequency speckle and intense peaks. However, it incurs higher pixel-wise error (MAE = 0.09). This elevated MAE manifests the double-penalty effect typical of stochastic downscaling: synthesized high-frequency details, while physically realistic in texture, misalign with ground-truth spatial coordinates. Although Minkowski-constrained UNets aim to bridge this gap by preserving exact spatial alignment, their inability to visually reproduce structural fidelity highlights a limitation. Future work must refine the weighting and optimization of integral-geometric constraints to match diffusion model quality without the associated computational sampling costs.

\section{Discussion and Conclusions}

This work establishes a rigorous framework for integrating non-differentiable integral-geometric domain knowledge into gradient-based training loops via learned surrogates, and demonstrates it on the EUMETNET OPERA radar dataset. We successfully trained a Lipschitz-regularized neural emulator to accurately map physical precipitation fields to their corresponding Minkowski functionals, structurally enforcing critical inductive biases such as monotonicity and the isoperimetric inequality. While incorporating differentiable surrogates into deep learning is a rapidly evolving frontier \citep{rehmann2025surrogatebaseddifferentiablepipelineshape, nguyen2026differentiablesurrogatedetectorsimulation}, our empirical results highlight fundamental challenges in transitioning from accurate metric emulation to effective generator optimization.

The constrained Lip-CNN achieved a 0\% isoperimetric violation rate, definitively eliminating the physically impossible geometries produced by unconstrained baselines. However, our super-resolution experiments reveal a fundamental trade-off: hard Lipschitz constraints, while essential for preventing optimization collapse, over-smooth the latent representation. The backpropagated error signals consequently lack the localized precision required to steer fine-scale spatial features; the surrogate recognizes macroscopic geometry but cannot provide the sharp, pixel-level corrective gradients needed to define distinct convective boundaries. Emerging libraries for generalized automatic differentiation of discrete operations \citep{paulus2026softjaxsofttorchempowering} offer complementary analytical tools, but sequential element-wise relaxations for complex topological metrics remain prone to gradient attenuation — reinforcing the case for a top-down learned surrogate once the gradient transfer bottleneck is resolved.

Ultimately, this framework addresses a ubiquitous requirement in Earth system science: the need to model fields where geometrical accuracy is encoded in coherent structures rather than independent grid points. For applications like statistical downscaling, where the spatial organization of extreme events directly dictates localized climate impacts, our findings underscore that formulating mathematically consistent loss functions is only half the solution. Future work must couple these rigorous topological constraints with stochastic generative architectures — such as denoising diffusion probabilistic models — to combine the strict integral-geometric guarantees of the Minkowski image loss with the textural synthesis capabilities required for full morphological realism.

\section*{Acknowledgments}
The authors thank L. Räss for technical assistance. We also acknowledge D. Nerini, D. Domeisen, V. Chavez, L. Moret, and O. Miralles for their contributions to the initial conceptualization of this work.  

\section*{Funding Statement}
F.Q., and T.B. acknowledge support from the Swiss National Science Foundation (SNSF) under Grant No. 10001754 (``RobustSR'' project).

\section*{Competing Interests} 
None.

\section*{Data Availability Statement}
The OPERA radar data used in this study are openly available under the CC BY license. Data can be accessed via the EUMETNET Open Radar Data (ORD) API. Access requires IP whitelisting; requests should be directed to the OPERA support team (\url{support.opera@eumetnet.eu}). Comprehensive documentation for data discovery and retrieval is available at \url{https://github.com/EUMETNET/openradardata-documentation}.

\section*{Ethical Standards}
The research meets all ethical guidelines, including adherence to the legal requirements of the study country. 

\section*{Author Contributions}
Conceptualization: F.Q., R.C., T.B. Methodology: F.Q., R.C. Software: F.Q. Formal analysis: F.Q., R.C. Investigation: F.Q. Data curation: F.Q. Writing - original draft: F.Q. Writing - review \& editing: F.Q., T.B. Visualization: F.Q. Supervision: T.B. All authors approved the final submitted draft.

\section*{Code Availability Statement}
The code implementing the full pipeline --- including the Minkowski functional emulator, 
the analytical approximation, and the super-resolution training and evaluation 
framework --- will be made publicly available through Zenodo and GitHub upon acceptance of this manuscript.

\bibliographystyle{apalike}
\bibliography{ref.bib}

\newpage
\appendix


\section{Minkowski functionals foundations}\label{sec:minkowski}

Let $K\subset \mathbb R^2$ be a compact, convex set. The Steiner formula \citep{schneiderStochastic2008} states that the area of the set of points within a distance $r$ from $K$ (for sufficiently small $r \ge 0$) is given by: 
\begin{equation}\label{eqn:steiner}
A_{K\oplus r} = A_K + P_K r + \chi_K \pi r^2 
\end{equation} 

where $A_K$ is the area, $P_K$ is the perimeter, and $\chi_K$ is the Euler characteristic of $K$. This formula holds not only for convex sets, but for any set with positive reach \citep{thale502008}, a condition satisfied by sufficiently regular boundaries. Equation~\eqref{eqn:steiner} then holds for all $r$ between 0 and the reach of $K$.

Equation~\eqref{eqn:steiner} is a special case of the general Steiner formula in dimension $n\in \mathbb N$, where $K \subset \mathbb R^n$ is compact and convex \citep[][Equation~1.2]{schneiderStochastic2008}. In the general formulation, the $n$-volume of the dilated set $K\oplus r$ may be expressed as an $n$-degree polynomial in $r > 0$, and the coefficients of this polynomial define the $n+1$ Minkowski functionals of $K$.

\section{Topological Filtering via Persistent Homology}

\subsection{Stability and Feature Significance}
While the Euler characteristic $\chi$ captures connectivity, the raw count of connected components is sensitive to noise; a single high-intensity pixel can create a spurious component, introducing discontinuities. To rigorously separate "signal" (stable storm cells) from "noise" (transient artifacts), we employ persistent homology.

Let $X$ be a smooth function on domain $D$. We examine the excursion sets $E_u = \{s \in D : X(s) > u\}$. Components are ``born'' at local maxima and ``die'' when they merge with older components at saddle points (as $u$ decreases). This lifecycle is captured by persistence pairs $(t_{\mathrm{birth}}, t_{\mathrm{death}})$.

\begin{proposition}\label{prp:persistence}
For any threshold $u$, the number of connected components of the excursion set above level $u$ equals the number of 0-dimensional persistence pairs such that $t_{\mathrm{death}} \leq u < t_{\mathrm{birth}}$.
\end{proposition}

\begin{proof}[Proof of Proposition~\ref{prp:persistence}]
    The claim follows almost immediately from the fundamental lemma of persistent homology \citep[][p.~181]{edelsbrunnerComputational2010}.
    The rank of the 0$^{\mathrm{th}}$ persistent homology group of the excursion set $\mathsf{H}_0(E_u)$ is given by the Betti number $\beta_0^{u,u} = \mathrm{rank}\,\mathsf{H}_0(E_u) = \{\text{\# connected components of }E_u\}$. Now, as a consequence of the fundamental lemma, $\beta_0^{u,u}$, which is the number of 0-dimensional homology classes alive at $u$, may be expressed as the number of 0-dimensional homology classes that are born before $u$ (at some $t_{\mathrm{birth}}>u$) and die after $u$ (at some $t_{\mathrm{death}}\leq u$).
\end{proof}

To stabilize the learning signal, we filter the persistence diagram by discarding pairs with a lifetime $t_{\mathrm{birth}} - t_{\mathrm{death}} < \epsilon$. This excludes points close to the diagonal of the persistence diagram, which correspond to low-persistence features indistinguishable from noise. By training on the filtered counts, we ensure the neural network focuses on learning robust, global topological structures rather than fitting numerical irregularities.

\subsection{Numerical Implementation and Target Generation} \label{appendix:gamma-targets}

We implement this filtration using the \texttt{GUDHI} library. Since standard algorithms compute sublevel set persistence, we input the negated precipitation field $-X$ to recover the superlevel set features required for precipitation clusters.

We fix the noise threshold at $\epsilon = 0.05\,\mathrm{mm\,h^{-1}}$. This value was determined by analyzing the distribution of feature lifetimes on a validation set. The persistence histogram exhibits a clear separation between high-frequency numerical noise (clustering near the diagonal of the persistence diagram) and robust topological features (the heavy tail). The chosen threshold $\epsilon$ corresponds to the ``elbow'' point of this distribution, ensuring that we filter out transient artifacts while retaining physically significant convective structures.

The counting process is fully vectorized: for any given quantile level $u$, we sum the significant persistence pairs that satisfy the condition in Proposition~\ref{prp:persistence}. Additionally, we explicitly handle the infinite-persistence feature (representing the global background component); this feature is included in the count only for low-intensity thresholds ($u \leq 0.01\,\mathrm{mm\,h^{-1}}$) to ensure consistent topology at the field boundaries.

To generate the complete $\boldsymbol{\gamma}$-matrix for the entire dataset, we execute a rigorous offline target generation pipeline. This process extracts a four-dimensional geometric summary for every sample: area, perimeter, number of connected components ($\beta_0$), and number of topological holes ($\beta_1$). 

First, we establish global physical intensity thresholds to ensure the evaluation levels are representative of the observed climatology. We compute an empirical cumulative distribution function by randomly sampling up to $5 \times 10^7$ precipitating pixels (exceeding the drizzle threshold of $0.1\,\mathrm{mm\,h^{-1}}$) strictly from the training split. The user-defined quantile levels are then mapped to these empirical physical thresholds.

The persistent homology computations are fully vectorized across these threshold domains. Using the persistence diagram obtained from the cubical complex, we isolate both 0-dimensional and 1-dimensional persistence pairs. For a given threshold $u$, a topological feature is counted if its birth is greater than or equal to $u$, its death is strictly less than $u$, and its total persistence (lifetime) exceeds the empirically defined noise thresholds (which can be tuned independently for $\beta_0$ and $\beta_1$). 

Concurrently, the extensive geometric measures are computed. Given a physical radar resolution of $2.0 \times 2.0\,\mathrm{km^2}$, the area at each threshold $u$ is calculated simply as the sum of the binary excursion set multiplied by the pixel area ($4.0\,\mathrm{km^2}$). To calculate a mathematically robust perimeter, we apply the marching squares algorithm to the binary masks at a 0.5 contour level. The total perimeter is derived by integrating the Euclidean lengths of the resulting sub-pixel contour segments and scaling by the spatial resolution. 

Finally, to manage the substantial combinatorial overhead of these exact topological and geometric evaluations, the target generation is executed offline using block-wise multiprocessing. The resulting multidimensional $\boldsymbol{\gamma}$-vectors are stored persistently in compressed Zarr arrays, which allows for highly efficient, chunked data loading within the downstream super-resolution optimization loop without bottlenecking the GPU accelerators.


\section{Analytical Approximation of Euler Characteristic}
\label{Appendix:Chi}

The pronounced performance gap in predicting the number of connected components ($R^2_{\mathrm{CC}}$) stems from a fundamental theoretical distinction in what the two surrogates actually evaluate. The analytical approximation computes the two-dimensional Euler characteristic ($\chi$), which is topologically defined as the number of connected components ($\beta_0$) minus the number of topological holes ($\beta_1$). Consequently, the analytical method yields a net structural scalar ($\chi = \mathrm{CC} - \mathrm{H}$), whereas the neural emulator is trained to exclusively regress the pure connected component count ($\mathrm{CC}$) derived from the exact ground truth. In complex precipitation fields, the prevalence of topological holes heavily skews the Euler characteristic, directly causing the severe degradation observed in the $R^2_{\mathrm{CC}}$ metric. Crucially, because the analytical formulation relies on localized morphological operations and continuous integrations, it is mathematically intractable to untangle the $\mathrm{CC}$ and $\mathrm{H}$ variables from the aggregated scalar $\chi$. This structural entanglement represents a significant limitation of the analytical surrogate, underscoring the critical advantage of the neural emulator, which can isolate and map specific geometric invariants without topological interference.

\section{Evaluation of Minkowski Functionals Emulations} 

\subsection{Qualitative analysis of Emulated Minkowski Functionals} \label{appendix:eval_emu}

\begin{figure}
    \centering
    \includegraphics[width=\linewidth]{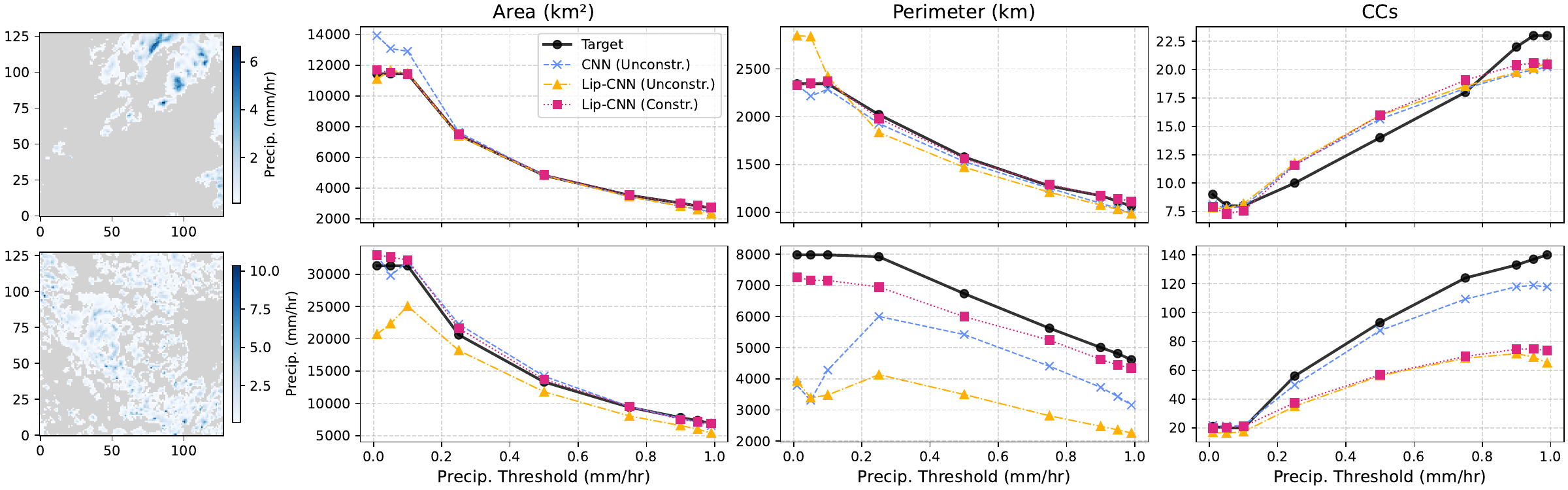}
    \caption{\textbf{Qualitative prediction analysis.} Comparison of ground truth (solid) and predicted (dashed) Minkowski functionals. \textbf{(Top)} A well-behaved storm where all models converge. \textbf{(Bottom)} A challenging, fragmented case. \textit{Unconstrained} models (blue/orange) predict geometrically impossible perimeters (violating monotonicity), whereas the \textit{Constrained} model (magenta) maintains consistency by design}
    \label{fig:storm-gamma-comparison}
\end{figure}

Figure \ref{fig:storm-gamma-comparison} illustrates the emulation behaviors on specific samples. For standard storm structures (top row), all models capture the geometry with high fidelity. However, divergences appear in complex cases (bottom row). Notably, while the \textit{vanilla CNN} predicts the area accurately, it fails to model the perimeter relationship correctly. The \textit{constrained} model, by virtue of the conditioning $\hat{P} = f(\hat{A}, \hat{r})$, maintains a consistent perimeter curve. Interestingly, for this specific high-fragmentation case, the vanilla CNN slightly outperforms the Lipschitz models on the Connected Components (CC) metric. This local discrepancy suggests that the smoothing induced by spectral normalization may occasionally suppress the detection of extremely small, grid-scale fragments, a known trade-off of Lipschitz regularization.

\subsection{Diagnostic Inversion and Spectral Analysis via Radially Averaged Power Spectral Density} \label{appendix:eval_inv}

\begin{figure}
    \centering
    \includegraphics[width=\linewidth]{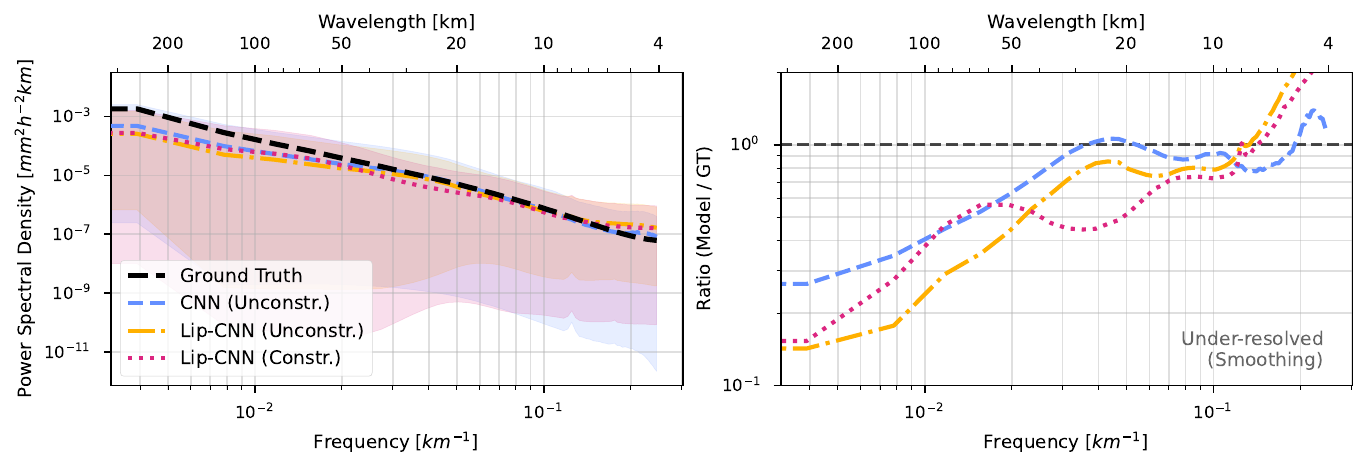}
    \caption{\textbf{Spectral Validation.} (Left) Radially Averaged Power Spectral Density comparing the energy cascade of the ground truth (black) against the reconstructions. (Right) The spectral ratio $S_{\text{model}}(k) / S_{\text{ref}}(k)$. Ideal performance corresponds to $y=1$}
    \label{fig:raspd}
\end{figure}

To verify that the emulator provides informative gradients, we perform a diagnostic inversion, steering a Gaussian noise vector $\mathbf{x_0}$ toward a target $\boldsymbol{\gamma}_{\text{true}}$ by minimizing the geometric discrepancy: 
\begin{equation} \label{eq:inversion} 
\mathbf{x}^* = \arg \min_{\mathbf{x}} \left( || \log(\hat{\boldsymbol{\gamma}}_{\theta}(\mathbf{x}) + \mathbf{1}) - \log(\boldsymbol{\gamma}_{\text{true}} + \mathbf{1}) ||_2^2 + \lambda_\mathrm{TV} \mathcal{R}_\mathrm{TV}(\mathbf{x}) + \lambda_{L^2} ||\mathbf{x}||_2^2 \right) 
\end{equation} 
where $\mathcal{R}_\mathrm{TV}$ is the anisotropic total variation norm ($\lambda_\mathrm{TV} = 10^{-5}$ and $\lambda_{L^2} = 10^{-6}$). 

To ensure the gradients do not introduce nonphysical noise, we validate the spectral fidelity of the reconstructions via the radially averaged power spectral density. To quantify the multiscale structural fidelity of the reconstructed fields, we compute the radially averaged power spectral density (RAPSD). Given a physical precipitation field $\mathbf{x} \in \mathbb{R}^{H \times W}$, we first apply a 2D Hanning window $w$ to mitigate spectral leakage at the boundaries. We then compute the 2D discrete Fourier transform (DFT), denoted as $\mathbf{\hat{x}}(k_x, k_y)$. The 2D power spectral density (PSD) is defined as:
\begin{equation} 
P(k_x, k_y) = \frac{1}{(HW)^2} |\mathcal{F}\{w \odot \mathbf{x}\}(k_x, k_y)|^2 
\end{equation}
Assuming statistical isotropy in the precipitation features, we reduce the 2D spectrum to a 1D profile by averaging the energy within annular bins of constant wavenumber magnitude $k = \sqrt{k_x^2+k_y^2}$. The RAPSD, $S(k)$, is calculated as the azimuthal mean:
\begin{equation} 
S(k) = \frac{1}{N_k} \sum_{(k_x, k_y) \in \mathcal{R}(k)} P(k_x, k_y) 
\end{equation}
where $\mathcal{R}(k)$ represents the set of discrete Fourier modes falling within the ring $k - \frac{1}{2} \leq \sqrt{k_x^2+k_y^2} < k + \frac{1}{2}$ and $N_k = |\mathcal{R}(k)|$ is the number of modes in that bin. This metric allows us to verify if the emulator reconstructs the correct power-law decay (energy cascade) characteristic of atmospheric turbulence.

Figure \ref{fig:raspd} shows that all three architectures reproduce the fundamental power-law decay characteristic of three-dimensional, moist atmospheric turbulence, with no single architecture emerging as a definitive winner. Remarkably, the vanilla CNN achieves spectral fidelity comparable to the constrained models across the inertial subrange. This suggests that the Minkowski image loss is sufficiently informative to enforce structural realism even in a standard unconstrained network, effectively steering it away from the ``checkerboard'' artifacts typically associated with pixel-wise optimization.

However, a common deficiency is observed at the largest scales: all architectures under-resolve the lower frequencies (wavelengths >50 km), exhibiting a spectral ratio below 1.0. This indicates that while the Minkowski image loss successfully constrains local texture and object shapes (high-frequency topology), recovering the total energy of large-scale synoptic features remains a challenge for optimization-based inversion starting from noise. This difficulty stems directly from the localized inductive bias of standard convolutional operators, which inherently struggle to coordinate coherent mass transport over large spatial domains. To address this fundamental limitation and fully capture large-scale spectral energy, future work must explore global integration architectures, such as Fourier neural operators (FNOs), as a necessary pathway to overcome the spatial constraints of conventional CNNs.

\subsection{Mechanistic Evaluation on Synthetic Fields} \label{appendix:eval_grads}

\begin{figure}
    \centering
    \includegraphics[width=\linewidth]{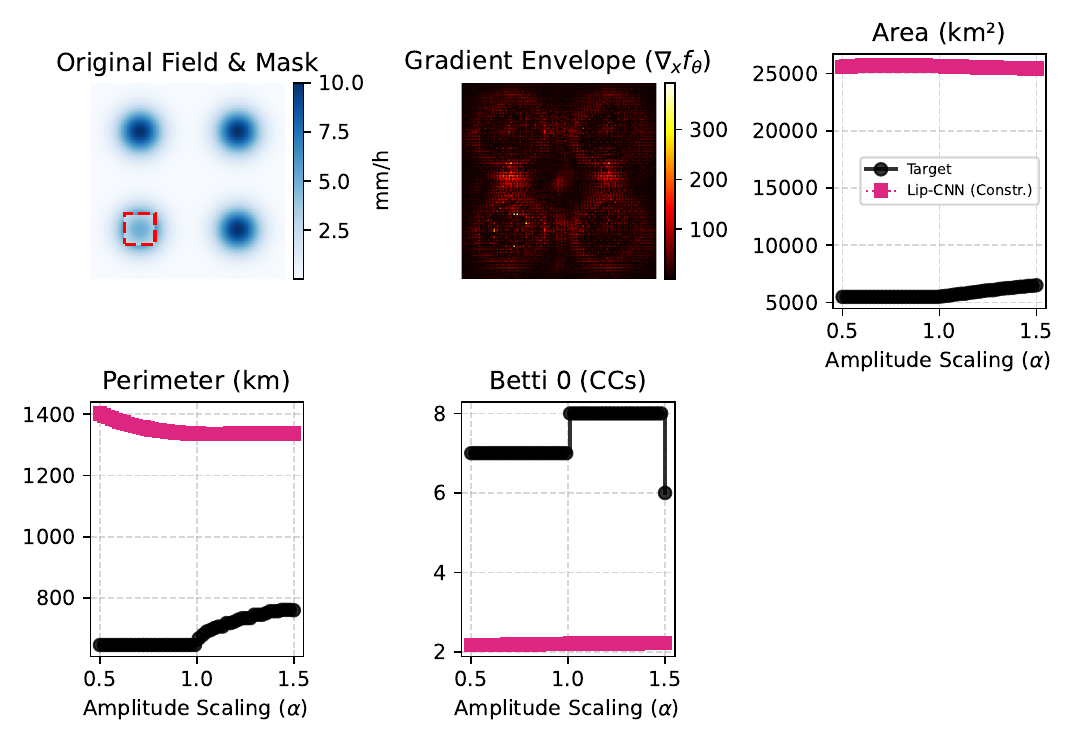}
    \caption{Mechanistic evaluation of the surrogate's sensitivity to localized physical perturbations. \textbf{(Left)} A synthetic multi-peak Gaussian precipitation field. The red dashed box denotes the spatial mask applied to a single convective cell, where the intensity is systematically scaled by an amplitude factor $\alpha$. \textbf{(Middle)} The spatial attribution map representing the surrogate's gradient envelope ($\nabla_{\mathbf{x}}f_{\theta}$), revealing a fragmented and noisy receptive field that lacks cohesive internal structure. \textbf{(Right)} The one-dimensional amplitude perturbation landscape for area, perimeter, and connected components (Betti 0). As the targeted cell's amplitude increases, the exact topological metrics (black) exhibit discrete step transitions when crossing the evaluation threshold. In contrast, the predictions from the constrained Lip-CNN (magenta) remain entirely static across the perturbation range, demonstrating that aggressive Lipschitz regularization over-smooths the latent space and renders the emulator blind to localized structural variations.}    
    \label{fig:mechproof}
\end{figure}

To rigorously evaluate the emulator's sensitivity to localized physical changes, we conduct a mechanistic proof using a synthetic multi-peak Gaussian field. As shown in Figure \ref{fig:mechproof}, a spatial mask is applied to a single convective cell, and its intensity is systematically perturbed by an amplitude scaling factor $\alpha$. As illustrated in the spatial attribution map, the surrogate's gradient envelope ($\nabla_{\mathbf{x}}f_{\theta}$) loosely identifies the active convective regions but exhibits a highly fragmented and noisy internal structure. More critically, the one-dimensional amplitude perturbation landscape exposes a fundamental failure in the surrogate's forward mapping. As the amplitude of the targeted cell varies, the exact integral-geometric targets (area, perimeter, and connected components) exhibit the expected discrete step transitions when the local intensity crosses the evaluation threshold. However, the constrained Lip-CNN predictions remain entirely static across the entire perturbation range. This extreme insensitivity demonstrates that the current regularization strategy---likely the interaction between aggressive spectral normalization and global pooling---severely over-smooths the latent representation. Consequently, the emulator becomes completely blind to isolated, localized intensity variations, rendering it incapable of providing the precise gradient signals required to optimize fine-scale precipitation structures.

\section{Super-resolution Training Protocols} \label{AppendixSR}

\subsection{Preprocessing}
To facilitate stable optimization across heavily skewed atmospheric variables and disparate physical units, we apply rigorous normalization protocols to the input tensors. The coarse-resolution precipitation inputs and high-resolution targets first undergo a log-linear transformation: variables are shifted via $x \mapsto \log(1+x)$ and subsequently divided by the global maximum of the log-transformed training set, constraining the values to a strictly positive $[0, 1]$ interval. For the stochastic diffusion baseline (DDPM), these normalized tensors are further linearly mapped to the $[-1, 1]$ domain to satisfy the standard Gaussian noise assumptions inherent to the reverse diffusion process. The static orographic forcing, represented by a high-resolution digital elevation model (DEM), is standardized using the domain-wide mean and standard deviation to yield zero mean and unit variance. The normalized DEM is concatenated channel-wise with the low-resolution precipitation field to explicitly condition the generative spatial downscaling on underlying topography. The target integral-geometric metrics ($\gamma$-vectors) are exclusively log-transformed to compress their dynamic range prior to loss computation.

During training, the dataset is dynamically augmented using random horizontal and vertical flips (each applied with a 50\% probability) to promote rotational and reflectional invariance. To address the severe class imbalance typical of meteorological radar data, where dry or drizzling pixels disproportionately dominate the spatial domain, the deterministic UNet architectures are trained using a stratified sampling strategy. We compute independent sample weights based on a binary wet/dry classification—defined by a localized maximum intensity threshold of 0.1 mm/h—and employ a weighted random sampler. This ensures the optimization landscape is adequately exposed to structurally complex, high-intensity convective events, preventing the network from collapsing to a climatological dry mean.

\subsection{Dynamic Loss Weighting and Trust Mechanism}
While the baseline log-space residual UNet is optimized solely via the MSE between the predicted and target log-transformed fields, the topologically constrained configurations augment this objective with the Minkowski image loss. The total loss function is defined as:
\begin{equation}
\mathcal{L}_{\text{total}} = \mathcal{L}_{\text{MSE}} + w_{\text{geom}}(t) \cdot w_{\text{trust}} \cdot \mathcal{L}_{\text{Minkowski}}
\end{equation}
where $w_{\text{geom}}(t)$ is a dynamic scaling factor governed by a cosine annealing schedule over the training epochs $t$. This allows the model to first learn large-scale mass conservation before focusing on fine-scale topological adjustments. Crucially, when utilizing the neural emulator, we introduce a dynamic trust mechanism $w_{\text{trust}} = \exp(-\tau \| \hat{\gamma}_{\text{true}} - \gamma_{\text{true}} \|^2)$ that penalizes the geometric loss component when the emulator's localized approximation error on the ground truth is high, parameterized by $\tau = 0.005725$. This ensures the generator is not penalized by out-of-distribution emulator hallucinations.

\subsection{Trust coefficient computation}

To compute the penalty parameter $\tau$, we perform an offline calibration pass over the validation dataset to characterize the neural emulator's baseline error distribution. High-resolution ground truth precipitation fields are denormalized from the scaled training domain back to physical intensity values ($\mathrm{mm\,h^{-1}}$) and masked by the $0.1\,\mathrm{mm\,h^{-1}}$ drizzle threshold. These physical fields are processed by the frozen emulator to predict the multidimensional integral-geometric components. To match the scale of the objective function, the predictions undergo a logarithmic transformation, $\log(1+x)$, before computing the mean squared error against the exact ground truth $\boldsymbol{\gamma}$-vectors. We extract the 90th percentile of this error distribution across the true climatology. The coefficient $\tau$ is algebraically derived to aggressively attenuate the geometric loss signal for out-of-distribution structures, specifically calibrating the trust factor $w_{\text{trust}}$ to decay to a minimal weight of $0.1$ when the emulator's approximation error reaches this 90th percentile threshold.

\newpage
\section{List of hyperparameters}

\begin{table}[htbp]
    \centering
    \begin{minipage}[t]{0.48\linewidth}
        \centering
        \caption{Hyperparameters and configuration settings for the emulator and feature inversion experiments. Parameters with $^{(*)}$ are calibrated using the \texttt{Optuna} package \citep{akiba2019optunanextgenerationhyperparameteroptimization}.}
        \label{tab:hyperparameters_emulator}
        \resizebox{\linewidth}{!}{%
        \begin{tabular}{l l l}
            \toprule
            \textbf{Category} & \textbf{Parameter} & \textbf{Value} \\
            \midrule
            \multirow{5}{*}{\textbf{Data \& Preprocessing}}
            & Patch Size & $128 \times 128$ \\
            & Pixel Resolution & $2.0$ km \\
            & Quantile Levels ($Q$) & 0.01, 0.05, 0.1, \dots \\ 
            & Drizzle Threshold & $0.1$ mm/h \\
            & Persistence Threshold & $0.05$ mm/h \\
            \midrule
            \multirow{6}{*}{\textbf{Architecture (Emulator CNN)}}
            & Base Channel Width & 32 \\
            & Max Channel Width & 256 \\
            & Network Depth & 5 layers \\
            & Latent Dimension & 256 \\
            & Normalization & GroupNorm ($G=8$) \\
            & Activation & GELU \\
            \midrule
            \multirow{7}{*}{\textbf{Training (Emulator)}}
            & Optimizer & Adam \\
            & Learning Rate$^{(*)}$ & $7.75 \times 10^{-5}$ \\
            & Weight Decay$^{(*)}$ & $2.49 \times 10^{-6}$ \\
            & Batch Size & 128 \\
            & Max Epochs & 50 \\
            & Early Stopping & 15 epochs \\
            & LR Scheduler & ReduceLROnPlateau \\
            \midrule
            \multirow{3}{*}{\textbf{Loss (Emulator)}}
            & Weight $\lambda_{\text{Area}}$ & $3.0$ \\
            & Weight $\lambda_{\text{Perimeter}}$ & $1.0$ \\
            & Weight $\lambda_{\text{CC}}$ & $1.5$ \\
            \midrule
            \multirow{4}{*}{\textbf{Inversion / Analysis}}
            & Optimization Steps & 200 \\
            & Inversion LR & $0.1$ \\
            & Reg. $\lambda_{\text{TV}}$ & $1 \times 10^{-5}$ \\
            & Reg. $\lambda_{L^2}$ & $1 \times 10^{-6}$ \\
            \bottomrule
        \end{tabular}%
        } 
    \end{minipage}\hfill
    \begin{minipage}[t]{0.48\linewidth}
        \centering
        \caption{Hyperparameters and configuration settings for the super-resolution experiments. Parameters with $^{(*)}$ are calibrated using the \texttt{Optuna} package \citep{akiba2019optunanextgenerationhyperparameteroptimization}.}
        \label{tab:hyperparameters_sr}
        \resizebox{\linewidth}{!}{%
        \begin{tabular}{l l l}
            \toprule
            \textbf{Category} & \textbf{Parameter} & \textbf{Value} \\
            \midrule
            \multirow{4}{*}{\textbf{Architecture (UNet SR)}}
            & Model & LogSpaceResidualUNet \\
            & Input Channels & 2 \\
            & Output Channels & 1 \\
            & Downscaling Factor & $12.5\times$ \\
            \midrule
            \multirow{4}{*}{\textbf{Architecture (DDPM)}}
            & Model & Conditional ContextUNet \\
            & Condition Channels & 2 \\
            & Fwd Diff. Timesteps & 1000 \\
            & DDIM Inference Steps & 50 \\
            \midrule
            \multirow{8}{*}{\textbf{Training (SR — shared)}}
            & Optimizer & AdamW \\
            & Batch Size & 128 \\
            & Max Epochs & 25 \\
            & Early Stopping & 5 epochs \\
            & LR Scheduler & ReduceLROnPlateau \\
            & UNet LR$^{(*)}$ & Optuna-tuned \\
            & UNet WD$^{(*)}$ & Optuna-tuned \\
            & DDPM LR / WD & $1.0 \times 10^{-5}$ \\
            \midrule
            \multirow{5}{*}{\textbf{Loss (UNet SR)}}
            & Base Loss & MSE (log space) \\
            & Geo Loss Schedule & Cosine annealing \\
            & Geo Warmup & 5 epochs \\
            & Gradient Clip & Max norm $1.0$ \\
            & Trust Coeff. $\tau$ & $5.725 \times 10^{-3}$ \\
            \bottomrule
        \end{tabular}%
        } 
    \end{minipage}
\end{table}

\end{document}